\documentclass{article} 
\usepackage{iclr2026_conference, times}

\usepackage{amsmath, amssymb, amsthm}
\usepackage{natbib}
\usepackage{hyperref}
\usepackage{url}
\usepackage{booktabs}
\usepackage{microtype}
\usepackage{xcolor}
\usepackage{graphicx}
\usepackage{multirow}
\usepackage{makecell}
\usepackage{wrapfig}
\usepackage{caption}

\newtheorem{theorem}{Theorem}[section]
\newtheorem{proposition}[theorem]{Proposition}

\newcommand{\R}{\mathbb{R}}
\newcommand{\Sp}{\mathbb{S}}

\newcommand{\cY}{\mathcal{Y}}
\newcommand{\cM}{\mathcal{M}}
\newcommand{\cD}{\mathcal{D}}
\newcommand{\SSW}{\mathrm{SSW}_2}
\newcommand{\SSWsq}{\mathrm{SSW}_2^2}

\newcommand{\calC}{C_\alpha}
\newcommand{\calCamb}{C^{\mathrm{amb}}_\alpha}
\newcommand{\taua}{\tau_\alpha}

\newcommand{\dM}{d_{\cM}}

\title{Manifold Constrained Conformal Prediction for Spatial Events}

\author{%
  Collin Nill \\
  Department of Statistics \\
  University of Connecticut\\
  Storrs, CT 06269 \\
  \texttt{collin.nill@uconn.edu} \
 \And
  Trevor A. Harris \\
  Department of Statistics \\
  University of Connecticut\\
  Storrs, CT 06269 \\
  \texttt{trevor.a.harris@uconn.edu} \
  \And
  Jason Adams \\
  Sandia National Laboratories\\
  Albuquerque, NM 87123 \\
  \texttt{jradams@sandia.gov} \
}

\usepackage{fancyhdr}
\fancypagestyle{plain}{
  \fancyhf{}
  \fancyhead[L]{\footnotesize Preprint}
  \fancyhead[R]{\footnotesize \today}
  
}
\iclrfinalcopy
\begin{document}
\maketitle

\begin{abstract}
We introduce a new conformal prediction method that constructs calibrated prediction sets over collections of spatial events, such as tropical cyclone genesis and earthquake locations. Forecasting natural hazards has become increasingly important, due to their significant economic impact, and quantifying the uncertainty of predictions is critical for accurate risk assessment. Our approach works by representing spatial point clouds as empirical measures so that we can score them using (sliced) Wasserstein distance, then constraining the resulting distribution-valued prediction set to be supported only near the training data manifold. We derive a coverage lower bound for the intersected sets and show that, in practice, this gap can be made small through a simple data-adaptive selection criterion. Because the resulting set is not analytically tractable, we introduce a modified flow-based sampling procedure, which allows us to represent and apply these prediction sets in practice as ensembles. Numerical experiments on synthetic data, tropical cyclone genesis, and earthquake occurrences show that our method achieves near-nominal coverage, with significantly lower energy distance and manifold distance than highest predictive density region (HDR) baselines along with generative model baselines.
\end{abstract}

\section{Introduction}
Natural disasters are among the most consequential and difficult-to-predict phenomena in Geoscience. Tropical cyclones, earthquakes, and wildfires collectively account for trillions of dollars in economic losses and hundreds of thousands of deaths each year \citep{merz2020impact, beven2018epistemic}. These hazards manifest as discrete, spatially organized configurations of events rather than as continuous fields. For example, the genesis locations of tropical cyclones in a given season, epicenters of earthquake sequences, ignition points of wildfire outbreaks. The number, positions, and spatial arrangement of these events are all meaningful, and the differences are particularly consequential for compound and extreme hazards, where dependence structure and spatial spread drive the largest risks \citep{davison2015statistics, zscheischler2018future, abbaszadeh2022perspective}. Uncertainty quantification for these phenomena cannot be reduced to scalar or low-dimensional summaries.

Modern forecasting systems are increasingly capable of operating at the level of event configurations. Climate models, stochastic simulators, and deep learning approaches all produce distributions over possible configurations, either through parameterized likelihood maps or through stochastic ensembles \citep{leutbecher2008ensemble, jordan2011operational, zhou2022neural, mukherjee2025neural}. Figure~\ref{fig:TEH} shows examples of such configurations for tropical cyclone genesis and earthquake events. None of these approaches are guaranteed to produce calibrated uncertainty, i.e. taking the 90th percentile of an ensemble does not automatically imply 90\% coverage of the true event cloud. The many interacting processes driving natural hazards make the data-generating mechanism effectively un-specifiable, so any model will carry residual miscalibration. What is required is a calibration procedure that operates post-hoc on any forecasting model and guarantees coverage regardless of its form.

\begin{figure}[t]
  \centering
  \includegraphics[width=\textwidth]{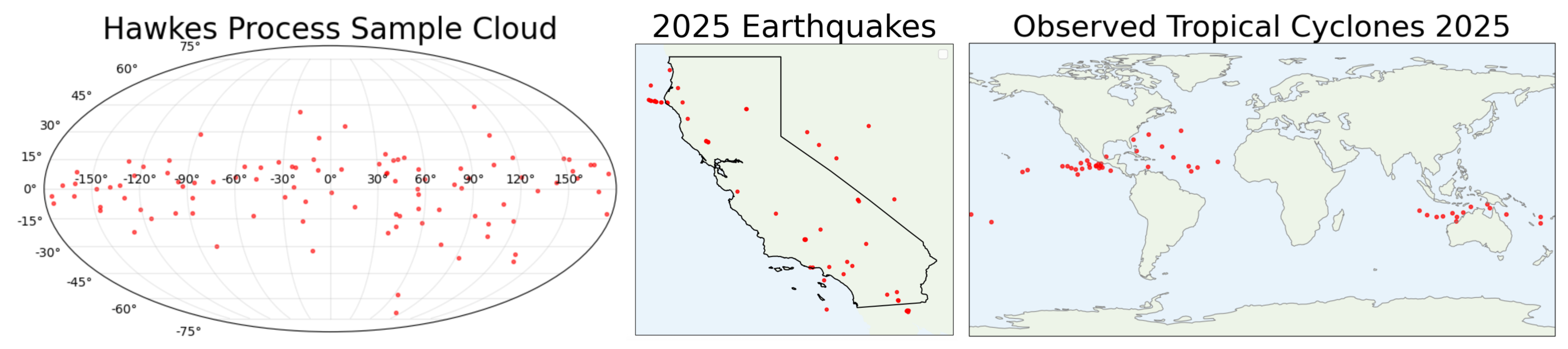}
  \caption{Examples of event clouds on the sphere. \textit{Right:} tropical cyclone genesis
  locations from one season \citep{gahtan2026international}, concentrated in tropical
  oceanic basins. \textit{Center:} significant earthquake epicenters from the NGDC/WDS
  database \citep{usgs_comcat_2026}, tracing the global fault network. \textit{Left:} a simulated
  calibration ensemble from a Hawkes process on the sphere (Section~\ref{apx:PP}).
  Each panel is a finite point cloud on $\Sp^2$ of variable cardinality.}
  \label{fig:TEH}
\end{figure}

Conformal prediction \citep{vovk2005algorithmic} provides exactly this guarantee under exchangeability \citep{shafer2008tutorial}. The prediction region is a sublevel set of a conformity score $S(y, \hat{y})$, and the coverage guarantee $\mathbb{P}(Y_{n+1} \in C_\alpha) \geq 1 - \alpha$ follows without distributional assumptions on the model or the data. The standard approach defines $S$ as a norm of the residual $r = y - \hat{y}$, which is not well-defined for point cloud predictions since clouds have variable cardinality, subtraction has no natural meaning, and scalar summaries such as mean location or total count can be identical for clouds with entirely different spatial organization.

We propose instead to carry out the comparison at the level of probability measures. Representing a point cloud $y$ as its empirical measure $\mu_y$ maps the conformity scoring problem onto distributional comparison, for which optimal transport provides the appropriate geometry \citep{villani2009optimal, peyre2019computational}. A conformity score based on $W_2$ produces prediction regions that are metric balls in $(P_2(\Sp^2), W_2)$, i.e. the set of all distributions within transport cost $\tau_\alpha$ of the predicted measure. The spherical sliced Wasserstein distance \citep{bonet2023spherical} provides a tractable approximation at $O(n \log n)$ cost. The idea of using OT distances as conformity scores has been explored concurrently for multivariate Euclidean outputs \citep{thurin2025optimal, klein2025multivariate} and graph-valued outputs \citep{melo2025conformal}; the present work addresses the space of probability measures on $\Sp^2$, where spherical geometry and variable cardinality require different treatment.

Spatially adaptive conformal methods \citep{tibshirani2019conformal,
barber2023conformal} conformalize scalar residuals at individual locations and
do not extend to the setting where the entire time interval point cloud
configuration which is the prediction target. Bayesian spatial point process
posteriors \citep{baddeley2015spatial} provide an alternative source of
uncertainty sets but require specifying a parametric generative model and
do not offer post-hoc coverage guarantees for an arbitrary forecasting
system. Within the conformal literature, the concurrent OT-based methods of
\citep{thurin2025optimal} and \citep{klein2025multivariate} operate on
multivariate Euclidean outputs and do not accommodate the geodesic geometry
of $\Sp^2$ or variable-cardinality clouds; \citep{melo2025conformal} use
Gromov--Wasserstein distance for graphs, a structurally different object
from empirical measures on the sphere. Consequently, the natural comparison
class for our method is HDR style region estimators applied to the forecast
probability grid, which represent the current practical standard for spatial
hazard uncertainty quantification.

Prediction sets defined as balls in $P_2(\Sp^2)$ are, however, far too vast to be useful \citep{villani2009optimal}. Without additional structure they contain physically implausible configurations, and unconstrained gradient flows in this space will generate off-manifold samples even when the target distance is correct \citep{ambrosio2005gradient}. Building on the flow-based conformal predictive distribution framework of \citep{harris2026flow}, we introduce a constrained gradient flow that simultaneously targets the SSW prediction set boundary and the data manifold, producing an ensemble of boundary configurations that are both calibrated and geophysically coherent. Experiments show that the resulting regions achieve nominal coverage, score well under proper distributional scoring rules, and concentrate on physically plausible configurations in a way the unconstrained ambient ball does not.

\section{Background}
 
\subsection{Conformal Prediction}
 
Conformal prediction \citep{vovk2005algorithmic} constructs prediction regions with finite-sample coverage guarantees under exchangeability. Given calibration pairs $\{(X_i, Y_i)\}_{i=1}^n$, a conformity score $S(Y, f(X))$ measures the discrepancy between the observed target and the model prediction. The calibration scores $S_i = S(Y_i, f(X_i))$ define the empirical $(1-\alpha)$-quantile
\begin{equation}
  \taua = S_{(\lceil (n+1)(1-\alpha) \rceil)},
\end{equation}
and the split conformal prediction region for a new test covariate $X_{n+1}$ is
\begin{equation}
  C_\alpha(X_{n+1}) = \{Y \in \mathcal{Y} : S(Y, f(X_{n+1})) \leq \taua\}.
\end{equation}
Under exchangeability, $\mathbb{P}(Y_{n+1} \in C_\alpha(X_{n+1})) \geq 1 - \alpha$. The framework extends to covariate shift \citep{tibshirani2019conformal}, distribution drift \citep{barber2023conformal}, functional data \citep{lei2015conformal, diquigiovanni2022conformal}, and flow-based predictive distributions \citep{harris2026flow}. The choice of conformity score determines the shape of the prediction region and how well it adapts to the problem geometry.
 
\subsection{Wasserstein Distances on the Sphere}
\label{sec:background_wass}
 
Let $P_2(\mathcal{Z})$ denote the space of Borel probability measures on a metric space $\mathcal{Z}$ with finite second moments. Optimal transport defines a geometry on $P_2(\mathcal{Z})$ by measuring the minimum cost of rearranging the mass of one measure to match another \citep{villani2009optimal}. For measures $\mu, \nu \in P_2(\mathcal{Z})$, the $p$-Wasserstein distance is
\begin{equation}
  W_p(\mu, \nu) = \left(\inf_{\pi \in \Pi(\mu,\nu)} \int_{\mathcal{Z} \times \mathcal{Z}}
  d(x,y)^p \, d\pi(x,y)\right)^{1/p},
\end{equation}
where $\Pi(\mu,\nu)$ is the set of couplings with marginals $\mu$ and $\nu$ and $d$ is the metric on $\mathcal{Z}$. For $\mathcal{Z} = \Sp^2$ the cost $d(x,y)$ is the geodesic (great-circle) distance. The space $(P_2(\mathcal{Z}), W_p)$ is itself a metric space, and $W_p(\mu_n, \mu) \to 0$ if and only if $\mu_n \to \mu$ weakly and the $p$-th moments converge \citep{villani2009optimal}. A conformity score based on $W_2$ therefore defines prediction regions that are metric balls in $(P_2(\Sp^2), W_2)$: the set of all distributions within transport cost $\taua$ of the predicted measure.
 
Computing $W_p$ exactly requires solving a linear program cubic in the number of support points. The sliced Wasserstein distance \citep{bonneel2015sliced} replaces this with a Monte Carlo average over one-dimensional projections,
\begin{equation}
  \mathrm{SW}_p^p(\mu,\nu) = \int_{\Sp^{d-1}} W_p^p(P_{\theta\,\#} \mu,\, P_{\theta\,\#}
  \nu)\, d\theta,
\end{equation}
where $P_\theta$ denotes projection onto direction $\theta$ and each one-dimensional $W_p$ is computed in $O(n \log n)$ by sorting. The integral is approximated by drawing directions $\theta$ uniformly from $\Sp^{d-1}$. The spherical sliced Wasserstein distance \citep{bonet2023spherical} adapts this construction to $\Sp^2$ by projecting onto great circles rather than Euclidean lines,
\begin{equation}
  \SSWsq(\mu,\nu) = \int_{V_{d,2}} W_2^2(P^U_{\,\#} \mu,\, P^U_{\,\#} \nu)\, d\sigma(U),
\end{equation}
where $V_{d,2}$ is the Stiefel manifold of orthonormal two-frames $(u_1, u_2) \in \R^{d \times 2}$ defining a great circle, $P^U$ maps each point $x \in \Sp^2$ to its coordinates in the plane spanned by $(u_1, u_2)$, and $\sigma$ is the uniform measure on $V_{d,2}$. In practice the integral is approximated by sampling frames uniformly. SSW inherits the metric properties of the sliced construction while respecting the curvature of $\Sp^2$. Throughout this paper, $\SSW(\mu,\nu)$ denotes the distance (the positive square root of $\SSWsq(\mu,\nu)$), and we use $\SSWsq$ only where the squared form appears explicitly.

\section{Method}
 
We now develop the conformal framework for point cloud predictions. We begin by defining the space of event clouds and the SSW conformity score, establish the ambient conformal region and its coverage guarantee, introduce an empirical surrogate for the data manifold, and then describe the constrained flow that samples the boundary of the manifold-restricted region.
 
\subsection{Ambient Conformal Region}
\label{sec:ambient}
 
Let $\mathcal{Y} = \bigcup_{m \geq 1} (\Sp^2)^m / \mathfrak{S}_m$ denote the space of finite unordered point clouds on $\Sp^2$, where $\mathfrak{S}_m$ is the symmetric group acting by permutation of coordinates. Each $Y \in \mathcal{Y}$ is a finite multiset of points on the sphere whose cardinality $|Y|$ may vary across observations. We represent each cloud by its empirical measure
\begin{equation}
    \mu_Y = \frac{1}{|Y|} \sum_{y \in Y} \delta_y \in P_2(\Sp^2),
\end{equation}
which embeds $\mathcal{Y}$ into the Wasserstein space $P_2(\Sp^2)$ introduced in Section~\ref{sec:background_wass}. This embedding is what makes the comparison of clouds of different cardinalities well-posed. Rather than requiring a point-to-point correspondence, we compare the distributions of mass that the two clouds induce on $\Sp^2$.
 
Let $\{(X_i, Y_i)\}_{i=1}^n$ be exchangeable calibration pairs, where $X_i$ is the covariate, $Y_i \in \mathcal{Y}$ is the observed target cloud, and $f(X_i)$ is the model prediction. We define the conformity score as the SSW distance between the empirical measure of the observed target and the model prediction,
\begin{equation}
  S(Y, f(X)) = \SSW(\mu_Y, f(X)),
\end{equation}
computed with a fixed set of slicing frames drawn once and shared across all evaluations. Since $\SSW$ metrizes the same weak convergence topology on $P_2(\Sp^2)$ as $W_2$ \citep{bonet2023spherical}, $S(Y, f(X)) = 0$ if and only if $\mu_Y = f(X)$, and $S$ is large precisely when the two empirical measures are far apart in the Wasserstein sense, accounting for differences in the spatial arrangement, clustering, and effective support of the two clouds.
 
The calibration scores $S_i = S(Y_i, f(X_i))$ yield the conformal cutoff
\begin{equation}
    \taua = S_{(\lceil(n+1)(1-\alpha)\rceil)},
\end{equation}
and the ambient conformal prediction region for a new test covariate $X_{n+1}$ is
\begin{equation}
  \calCamb(X_{n+1}) = \bigl\{Y \in \mathcal{Y} : S(Y, f(X_{n+1})) \leq \taua\bigr\}.
  \label{eq:ambient_region}
\end{equation}
Viewed in $P_2(\Sp^2)$, this is the preimage under $Y \mapsto \mu_Y$ of the SSW ball of radius $\taua$ centered on $f(X_{n+1})$. Under exchangeability of the calibration and test pairs,
\begin{equation}
    \mathbb{P}\bigl(Y_{n+1} \in \calCamb(X_{n+1})\bigr) \geq 1 - \alpha.
\end{equation}
The scores $S_1, \ldots, S_n, S_{n+1}$ are exchangeable since the pairs $(X_i, Y_i)$ are exchangeable. By the standard split conformal argument \citep{vovk2005algorithmic}, the empirical $(1-\alpha)$-quantile $\taua$ of $S_1, \ldots, S_n$ satisfies $\mathbb{P}(S_{n+1} \leq \taua) \geq 1 - \alpha$, which is exactly the event $Y_{n+1} \in \calCamb(X_{n+1})$.
 
\subsection{Manifold Restriction}
\label{sec:manifold}
 
The ambient region $\calCamb(X_{n+1})$ is an SSW ball in $P_2(\Sp^2)$ and enjoys the standard split conformal coverage guarantee. It is, however, far larger than the support of physically realistic event clouds. As a preimage in $\mathcal{Y} = \bigcup_{m \geq 1} (\Sp^2)^m / \mathfrak{S}_m$, it contains clouds of arbitrary cardinality and spatial arrangement: configurations scattered uniformly over the globe, clouds with hundreds of events where observations have dozens, and clouds with modes at dynamically implausible locations. All of these are statistically admissible because they lie within SSW distance $\taua$ of the prediction, but they carry no physical content and make the prediction region uninterpretable for downstream use.
 
The data-generating process concentrates on a much smaller subset. Physical dynamics constrain event clouds to structures such as tectonic fault networks for earthquakes or climatological basin boundaries for tropical cyclones. We denote the support of the data-generating distribution by the compact manifold $\cM \subset \mathcal{Y}$ and define the manifold-restricted conformal region
\begin{equation}
  \calC^{\textrm{man}}(X_{n+1}) = \calCamb(X_{n+1}) \cap \cM.
  \label{eq:manifold_region}
\end{equation}
Since $Y_{n+1} \in \cM$ almost surely by assumption, the events $\{Y_{n+1} \in \calCamb(X_{n+1})\}$ and $\{Y_{n+1} \in \calC^{\textrm{man}}(X_{n+1})\}$ coincide almost surely, so the coverage guarantee of the ambient region holds without modification, i.e. $\mathbb{P}(Y_{n+1} \in \calC^{\textrm{man}}((X_{n+1})) \geq 1 - \alpha$.
 
In practice $\cM$ is unknown, so we approximate membership in $\cM$ via proximity to the training targets $\mathcal{D}_0 = \{Y_i^0\}_{i=1}^{n_0}$. Specifically, for any candidate cloud $Y$ we define the manifold discrepancy
\begin{equation}
  \dM(Y) = \min_{Y' \in \cD_0} \SSW(\mu_Y, \mu_{Y'}).
  \label{eq:manifold_discrepancy}
\end{equation}
This is the SSW distance from $\mu_Y$ to the nearest element of the training set in $P_2(\Sp^2)$. When $\dM(Y)$ is small, $Y$ is distributionally close to an observed configuration; when it is large, $Y$ occupies a region of $\mathcal{Y}$ far by the training data. We impose membership in the empirical surrogate of $\cM$ through the tube constraint $\dM(Y) \leq \varepsilon$ for a tolerance $\varepsilon > 0$. The empirical manifold-restricted prediction set becomes
\begin{equation}
  \calC(X_{n+1}) = \bigl\{Y \in \mathcal{Y} : S(Y, f(X_{n+1})) \leq \taua,\;
  \dM(Y) \leq \varepsilon \bigr\}.
  \label{eq:empirical_manifold_region}
\end{equation}
 
The coverage of $\calC(X_{n+1})$ depends on whether $Y_{n+1}$ falls inside the empirical tube $\{\dM(Y) \leq \varepsilon\}$. We define $\dM$ using the training targets $\mathcal{D}_0 = \{Y_i^0\}_{i=1}^{n_0}$, where $\{(X_i^0, Y_i^0)\}_{i=1}^{n_0}$ are the training pairs used to fit $f$, which are independent of the calibration set
$\{(X_i, Y_i)\}_{i=1}^n$ by assumption. Concretely,
\begin{equation}
    \dM(Y) = \min_{1 \leq i \leq n_0} \SSW(\mu_Y, \mu_{Y_i^0}).
\end{equation}
Since $\mathcal{D}_0$ is independent of $\{(X_i, Y_i)\}_{i=1}^n$, the manifold discrepancy $\dM(Y_{n+1})$ is independent of the calibration scores $S_1, \ldots, S_n$, and no splitting of the calibration data is required. However, because $\{\dM(Y) \leq \varepsilon\}$ is an empirical quantity and not the exact manifold $\cM$, coverage is not automatic. Proposition \ref{prop:manifold} bounds the coverage gap.
 \begin{proposition}[Manifold Coverage Lower Bound]
\label{prop:manifold}
Assume $Y_1^0,\dots,Y_{n_0}^0,Y_{n+1}$ are i.i.d.\ from a distribution supported on a bounded $\cM \subset \cY$, and that $\cD_0=\{Y_i^0\}_{i=1}^{n_0}$ is independent of the calibration sample $\{(X_i,Y_i)\}_{i=1}^n$. Let $A_1,\dots,A_{N_\varepsilon}$ be a finite measurable cover of $\cM$ such that
$$
\sup_{Y,Y' \in A_k} \SSW(\mu_Y,\mu_{Y'}) \le \varepsilon,
\qquad k=1,\dots,N_\varepsilon,
$$
and define $\beta_\varepsilon := \min_{1\le k\le N_\varepsilon} \mathbb{P}(Y \in A_k)$.
Then
$$
\mathbb{P}\bigl(Y_{n+1} \in \calC(X_{n+1})\bigr)
\ge
1-\alpha - N_\varepsilon(1-\beta_\varepsilon)^{n_0}.
$$
\end{proposition}
Proof in Appendix \ref{apx:proof_manifold}. This bound isolates the cost of enforcing an approximate manifold constraint. Compared to the usual split conformal guarantee, there is an additional failure probability $N_\varepsilon(1-\beta_\varepsilon)^{n_0}$, which is the chance that there are no training samples in the $\varepsilon$-scale region of $\mathcal M$ containing the test response. However, if $\mathcal D_0$ forms an adequate $\varepsilon$-resolution representation of $\mathcal M$, then coverage approaches $1-\alpha$ exponentially fast in the training size $n_0$. 

Proposition \ref{prop:manifold} also makes the role of $\varepsilon$ explicit in that a larger $\varepsilon$ reduces $N_\varepsilon$ and increases $\beta_\varepsilon$, both of which shrink the coverage gap towards $1 - \alpha$. However increasing $\varepsilon$ also expands $\calC$ toward the ambient ball $\calCamb$, which may not accurately represent $\mathcal{M}$. In practice, we just set $\varepsilon$ large enough to cover all training data which was found to work reasonably well, albeit conservatively (Section \ref{sec:coverage}). 
 
 
\subsection{Boundary Construction by Constrained Flow}
\label{sec:flow}
 
\citep{harris2026flow} introduces a flow-based construction for sampling the boundary of conformal prediction regions. Candidate distributions are evolved under a gradient ODE that drives the conformity score to the calibration quantile $\taua$. We extend this framework to incorporate the manifold constraint $\dM(Y) \leq \varepsilon$, yielding a flow that converges to the boundary of the empirical manifold-restricted region $\calC(X_{n+1})$ defined in equation~\eqref{eq:empirical_manifold_region}.
 
Let $Y(t) \in \mathcal{Y}$ be a time-indexed candidate cloud and let $y(t) \in \R^{3|Y|}$ denote the vectorization of its point coordinates. Define the two scalar quantities
\begin{equation}
    S(y) = S(Y, f(X_{n+1})) = \SSW(\mu_Y, f(X_{n+1})), \qquad d(y) = \dM(Y),
\end{equation}
with gradients $g_S(y) = \nabla_y S$ and $g_d(y) = \nabla_y d$ obtained by automatic differentiation through the SSW computation. We seek a velocity field $v(y) = \dot{y}$ such that the trajectory $Y(t)$ converges to the set $\{S = \taua,\, d \leq \varepsilon\}$.
We specify this by imposing target scalar dynamics on both quantities:
\begin{align}
  \frac{d}{dt} S(Y(t)) &= -\lambda_1\bigl(S(Y(t)) - \taua\bigr),
  \label{eq:dyn_S}\\
  \frac{d}{dt} \dM(Y(t)) &= -\lambda_2\bigl(\dM(Y(t)) - \varepsilon\bigr)_+,
  \label{eq:dyn_d}
\end{align}
where $\lambda_1, \lambda_2 > 0$ are decay rates and $(u)_+ = \max\{u, 0\}$. The first equation drives $S$ exponentially toward $\taua$ from any initial value; at steady state $S(Y) = \taua$, placing the cloud on the conformal boundary. The second equation drives $d$ toward $\varepsilon$ only when $d > \varepsilon$, acting as a one-sided restoring force that enforces the manifold tube without disturbing clouds already inside it. At steady state, $Y$ lies on the conformal shell and within the manifold tube, i.e. on the boundary of $\calC(X_{n+1})$.
 
By the chain rule, equations~\eqref{eq:dyn_S}--\eqref{eq:dyn_d} impose two linear constraints on $v(y)$:
\begin{equation}
  J(y)\, v(y) = -c(y), \qquad
  J(y) = \begin{pmatrix} g_S(y)^\top \\ g_d(y)^\top \end{pmatrix} \in \R^{2 \times 3|Y|},
  \quad
  c(y) = \begin{pmatrix} \lambda_1\bigl(S(Y) - \taua\bigr) \\
    \lambda_2\bigl(\dM(Y) - \varepsilon\bigr)_+ \end{pmatrix}.
  \label{eq:linear_constraint}
\end{equation}
Since $J$ has two rows and $3|Y|$ columns, the system $Jv = -c$ is underdetermined for any cloud with $|Y| \geq 1$. Among the infinitely many solutions we take the minimum norm solution, given by the Moore--Penrose pseudoinverse:
\begin{equation}
  v(y) =-J(y)^\top\{J(y)J(y)^\top\}^{+}c(y).
  \label{eq:velocity}
\end{equation}
Derivation in Appendix \ref{apx:velocity_field}.
The minimum-norm choice is natural here since it satisfies the two prescribed scalar dynamics while introducing no additional motion in the directions orthogonal to both $g_S$ and $g_d$, minimizing numerical drift away from the constraint surface.

The flow is integrated by forward Euler steps. After each step, every point in $Y$ is projected back onto $\Sp^2$ by normalization to ensure the trajectory remains on the sphere. The spherical projection introduces an error of order $\delta^2$ in the constraint residual at each step, which is corrected by the restoring dynamics at the next iteration.

The velocity field in Equation \ref{eq:velocity} allows us to sample along the boundary of a single $\alpha \in (0, 1)$ conformal prediction set. Using the same technique as in \citep{harris2026flow}, we can marginalize over a sequence of these samplers to produce a calibrated predictive distribution. 
\begin{equation} \label{eqn:conformal_distribution}
    P_x^{\mathrm{CPD}}(A) = \int_0^1 \nu_{x, \alpha}(A)d\alpha, \qquad A \in \mathcal{B}(\mathcal{Y}),
\end{equation}
where each $\nu_{x, \alpha}$ is the pushforward distribution of a base measure $\mu_x$ through the flow induced by the velocity field (Equation \ref{eq:velocity}). We take the $\mu_x$ to be the empirical distribution of the calibration data.

 
 

\section{Numerical Studies} \label{sec:numerical_studies}


We evaluate the coverage (Prop. \ref{prop:manifold}) and forecast skill (Equation \ref{eqn:conformal_distribution}) of our proposed approach on synthetic processes (Sect. \ref{sec:coverage}, \ref{sec:synthetic_data}) and real geophysical data (Sect. \ref{sec:tropical_cyclones}, \ref{sec:earthquakes}).

Synthetic data is generated from four different point processes: a homogeneous Poisson process (HPP), a misspecified spiral point process (Spiral), a three mode inhomogeneous Poisson processes (3-IHPP), and a Hawkes process. These processes represent increasing levels of spatial heterogeneity and dependence. For each process we generate 100 sample clouds for training and calibration data, and 50 clouds of test data. We use a simple 2D histogram model for prediction in each case. Full simulation settings are described in detail in Appendix ~\ref{apx:baselines}.


\subsection{Empirical Coverage Lower Bound} \label{sec:coverage}

We first evaluate the empirical coverage of our procedure under the four simulation settings: HPP, Spiral, 3-IHPP, and the Hawkes under a range of hyperparameter settings. 
\begin{figure}[h]
  \centering
  \includegraphics[width=\textwidth]{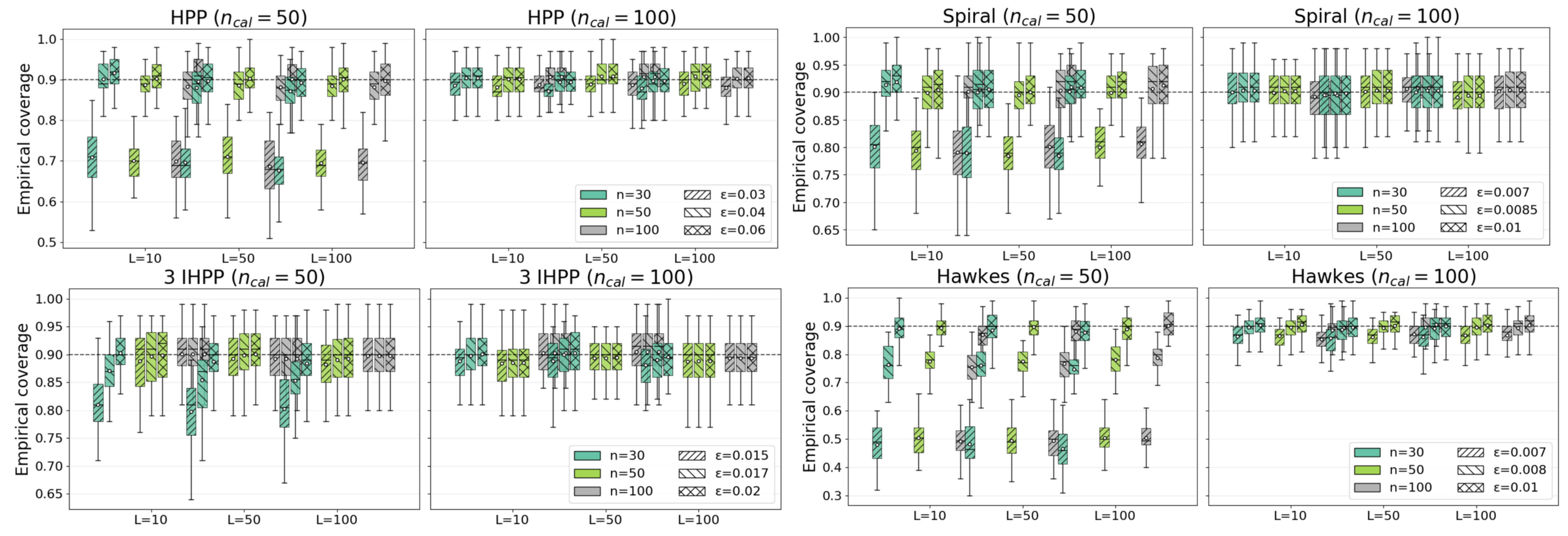}
  \captionsetup{font=small}
  \caption{Coverage by varying calibration size ($n_{cal}$), cardinality ($n$), slices in SSW distance ($L$), and manifold tolerance $(\varepsilon)$ over test simulations}
  \label{fig:coverage}
\end{figure}
Figure \ref{fig:coverage} shows the empirical coverage of our approach on each process under an increasing number of SSW slices ($L$), cardinality $n$, calibration samples $n_{cal}$, and empirical manifold tolerance $\varepsilon$. With enough calibration samples $(n_{cal} = 100)$, the number of slices $L$ and manifold tolerance $\varepsilon$ do not strongly impact coverage. However, the exact manifold tolerance value can strongly impact coverage. If the $\varepsilon$ is set too low, then our approach can severely under cover the target as implied by the lower bound in Proposition \ref{prop:manifold}. The number of slices $L$ has little to no effect, but should be set high enough to ensure stable estimation of SSW.

To further assess the effect of $\varepsilon$ on coverage, we fixed $L = 100,\ n = 100,\ n_{cal} = 100$ and computed the coverage of our method under of each process again under a dense sequence of $\varepsilon$ values. Figure \ref{fig:expCov} shows the coverage function of our method on each process, demonstrating the exponential coverage bound (Prop. \ref{prop:manifold}) and showing that if $\varepsilon$ is set high enough, then coverage is nearly exact. 
\begin{figure}[h]
  \centering
  \includegraphics[width=\textwidth]{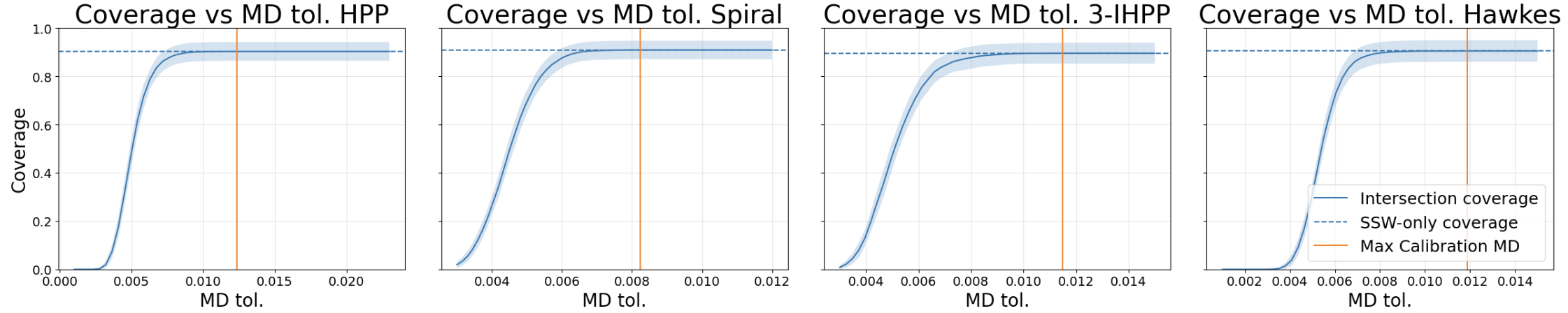}
  \captionsetup{font=small}
  \caption{Empirical coverage of the manifold constrained conformal prediction sets on each process. $n_{cal}$, $n$ and $L$ are all fixed at 100.}
  \label{fig:expCov}
\end{figure}
The vertical line in each panel of Figure \ref{fig:expCov} shows the smallest $\varepsilon$ value such that all training data is within $\varepsilon$ of the manifold estimated by the other training data. We use this as a data driven approach for selecting a conservative $\varepsilon$ value with nearly exact coverage in practice.

\subsection{Synthetic Data Evaluations} \label{sec:synthetic_data}
We use the four data generating mechanisms in Section \ref{sec:coverage}, HPP, Spiral, 3-IHPP, and Hawkes, to evaluate the predictive distribution produced by an unconstrained velocity field (VF) (\citep{harris2026flow}) and a manifold constrained velocity field (CVF) (Equation \eqref{eq:velocity}. For each process, we generate 50 training, calibration, and test point clouds with cardinality 100. For our predictive model, we construct a simple predictor by gridding the training observations into cells and normalizing the cell counts into probabilities, i.e. a 2D histogram (Appendix \ref{apx:density}).

We compare VF and CVF (Section \ref{sec:flow}) against five baseline approaches. These include a minimum area highest density region (HDR) estimator inspired by classical HDR and bayesian HPD sets (M1)\citep{Hyndman1996HDR, Polonik1997MinimumVolume}, 
a conformalized minimum area HDR using RAPS (M2) \citep{AngelopoulosBatesMalikJordan2021RAPS}, and three standard generative models: Diffusion (M3) \citep{Diffusion}, Normalizing Flow (M4) \citep{NormFlow}, and a Variational Auto Encoder (VAE) (M5) \citep{VAE}. The two HDR baselines, VAE (M5), and diffusion (M3) models sample directly from the outcome space, while VF, CVF, and the normalizing flow require initial samples from a discrete probability grid (Appendix \ref{apx:density}). This latter step introduces a discretization that the HDR, VAE, and Diffusion baselines do not require. Details of each method Appendix \ref{apx:baselines}.

\begin{table*}[h]
\centering
\captionsetup{font=small}
\caption{Simulation results across four settings. MinArea HDR (M1), Conformal MinArea HDR (M2), Diffusion (M3), Normalizing Flow (M4), VAE (M5), velocity field (VF), and constrained velocity field (CVF). Energy distance (ED) scaled by \(\times 10\) for IHPP-3 and Hawkes. Manifold distance (MD) is scaled by \(\times 100\). Standard deviations are reported in parentheses.}
\label{tab:simulation_results_updated}
\scriptsize
\setlength{\tabcolsep}{4pt}
\resizebox{\textwidth}{!}{
\begin{tabular}{llccccccc}
\toprule
Process & Metric & M1 & M2 & M3 & M4 & M5 & VF & CVF \\
\midrule

\multirow{4}{*}{HPP}
& \textbf{ED}
& 13.48 (24.73) & 13.75 (24.58) & 14.49 (23.64) & 14.29 (24.92) & 13.77 (24.73) & 1.04 (5.43) & 0.90 (5.89) \\
& {\normalfont\tiny\hspace{1em}Pred--Pred}
& 5.18 (0.00) & 5.64 (0.00) & 6.42 (0.01) & 5.11 (0.00) & 5.36 (0.01) & 33.12 (0.04) & 32.60 (0.04) \\
& {\normalfont\tiny\hspace{1em}True--Pred}
& 24.28 (15.07) & 24.65 (14.97) & 25.42 (14.48) & 24.65 (15.14) & 24.54 (15.07) & 32.05 (5.45) & 31.69 (5.73) \\
& \textbf{MD \(\times 100\)}
& 1.25 (0.09) & 1.27 (0.10) & 1.11 (0.08) & 1.11 (0.08) & 1.27 (0.11) & 1.76 (0.98) & 1.35 (0.58) \\
\midrule

\multirow{4}{*}{Spiral}
& \textbf{ED}
& 2.32 (1.01) & 1.34 (0.77) & 0.56 (0.52) & 1.03 (0.64) & 2.53 (1.48) & 0.51 (0.83) & 0.49 (0.83) \\
& {\normalfont\tiny\hspace{1em}Pred--Pred}
& 5.74 (0.01) & 5.61 (0.01) & 7.86 (0.01) & 7.74 (0.01) & 7.26 (0.01) & 4.62 (0.01) & 4.60 (0.01) \\
& {\normalfont\tiny\hspace{1em}True--Pred}
& 6.91 (0.65) & 6.35 (0.62) & 7.09 (0.44) & 7.27 (0.51) & 7.77 (0.89) & 5.45 (0.61) & 5.43 (0.62) \\
& \textbf{MD \(\times 100\)}
& 2.18 (0.12) & 1.86 (0.10) & 1.80 (0.24) & 2.35 (0.34) & 2.31 (0.44) & 0.99 (0.46) & 0.79 (0.29) \\
\midrule

\multirow{4}{*}{IHPP-3}
& \textbf{ED \(\times 10\)}
& 3.59 (9.07) & 0.09 (5.95) & 8.07 (15.42) & 14.03 (6.09) & 4.64 (7.74) & 0.06 (4.99) & -0.18 (4.94) \\
& {\normalfont\tiny\hspace{1em}Pred--Pred}
& 5.41 (0.01) & 5.64 (0.01) & 5.45 (0.01) & 5.29 (0.00) & 5.72 (0.01) & 5.40 (0.01) & 5.39 (0.01) \\
& {\normalfont\tiny\hspace{1em}True--Pred}
& 5.43 (0.65) & 5.37 (0.56) & 5.67 (0.89) & 5.89 (0.49) & 5.63 (0.56) & 5.24 (0.52) & 5.22 (0.52) \\
& \textbf{MD \(\times 100\)}
& 0.59 (0.09) & 0.67 (0.11) & 0.67 (0.15) & 1.17 (0.22) & 0.72 (0.14) & 0.56 (0.16) & 0.54 (0.18) \\
\midrule

\multirow{4}{*}{Hawkes}
& \textbf{ED \(\times 10\)}
& 3.61 (17.10) & 3.80 (17.86) & 12.51 (27.33) & 2.75 (15.57) & 3.73 (17.63) & 1.06 (14.52) & 1.09 (14.32) \\
& {\normalfont\tiny\hspace{1em}Pred--Pred}
& 5.21 (0.01) & 5.55 (0.01) & 5.45 (0.01) & 5.31 (0.01) & 5.19 (0.01) & 7.11 (0.01) & 7.11 (0.01) \\
& {\normalfont\tiny\hspace{1em}True--Pred}
& 6.44 (1.33) & 6.62 (1.34) & 7.00 (1.66) & 6.45 (1.26) & 6.44 (1.33) & 7.27 (1.17) & 7.27 (1.16) \\
& \textbf{MD \(\times 100\)}
& 0.76 (0.11) & 0.82 (0.13) & 0.70 (0.07) & 0.76 (0.06) & 0.74 (0.07) & 0.57 (0.19) & 0.55 (0.18) \\
\bottomrule
\end{tabular}
}
\end{table*}

We evaluate each predictive distribution by its energy distance \citep{rizzo2016energy} to the test data and by its manifold distance to the estimated data manifold (Section \ref{sec:manifold}). Energy distance is a proper scoring rule for probabilistic forecasts \citep{GneitingRaftery2007}. Because the outcomes are distributions, we use a Wasserstein energy distance \citep{sejdinovic2013equivalence} 
\begin{equation} \label{eqn:energy_distance}
   \text{ED}(X, Y) = 2\mathbb{E}[SSW(X,Y)]-\mathbb{E}[SSW(Y,Y')]-\mathbb{E}[SSW(X,X')] 
\end{equation}
where the standard $\ell^2$ norm base metric has been substituted for the spherical sliced Wasserstein (SSW) metric. We measure manifold distance using $\text{MD}(Y) = d_\mathcal{M}(Y)$ (Equation \ref{eq:manifold_discrepancy}).

Table~\ref{tab:simulation_results_updated} shows the average energy distance (ED) and manifold distance (MD) over the test data. CVF consistently achieves the lowest MD across each of the four settings except for the homogenous poisson process. This indicates that the CVF is both closer in distribution space to the target process and is closer to the data manifold, however, in reference to the homogenous poisson process the manifold method will not overlearn the manifold in instances of complete randomness. I.e., CVF is not achieving lower ED by simply shifting mass into disallowed regions. For example, by placing mass between the modes of the 3-IHPP to smooth estimation as the normalizing flow does. 

\subsection{Tropical Cyclones} \label{sec:tropical_cyclones}

We now apply our method to Tropical Cyclone genesis forecasting, which tries to predict where tropical cyclones are likely to originate each season \citep{QianGenesis}.
Tropical cyclone events are obtained from the IBTrACS database \citep{gahtan2026international}, which records the
time and location where each storm attains tropical cyclone strength and status \citep{QianGenesis}. 
We train a simple prediction model (Appendix ~\ref{apx:ModCyc}) on the 1980 to 1994 data to generate yearly global probabilistic forecast maps. We then use the 1995 to 2015 data as calibration for each conformal method (Sec. \ref{sec:synthetic_data}) and fit each of the baseline methods, VF, and CVF. 
For CVF, we set $\varepsilon$ large enough to fully cover the training data.
\begin{wraptable}[10]{r}{0.55\columnwidth}
    \vspace{-0\baselineskip}
    \captionsetup{font=small}
    \caption{Tropical cyclone energy distance metrics}
    \centering
    \setlength{\tabcolsep}{2.75pt}
    \begin{tabular}{lccccccc}
    \toprule
    Metric & M1 & M2 & M3 & M4 & M5 & VF & CVF \\
    \midrule
    \textbf{ED}\(\times 100\) & 6.78 & 9.16 & 0.95 & 2.90 & 2.48 & 2.52 & 2.39 \\
    {\normalfont\small\hspace{1em}Pred--Pred} & 0.08 & 0.07 & 0.08 & 0.07 & 0.11 & 0.08 & 0.08 \\
    {\normalfont\small\hspace{1em}True--Pred} & 0.12 & 0.13 & 0.09 & 0.10 & 0.11 & 0.10 & 0.10 \\
    \textbf{MD}\(\times 100\) & 1.19 & 1.01 & 1.53 & 2.85 & 1.23 & 1.05 & 0.49 \\
    \bottomrule
    \end{tabular}
    \label{tab:TCYC}
\end{wraptable}

Table \ref{tab:TCYC} shows the average energy distance (ED) (Equation \ref{eqn:energy_distance}), broken down by its constituent terms, and the average manifold distance ($\text{MD}$) over the test period (2015 - 2025). The  CVF method achieves the second lowest ED $(2.39)$ compared to either VF $(2.52)$ or any of the baseline methods other than Diffusion (M3). However, the decrease in energy distance for M3 is due to mode concentration and inflated variability off the manifold for individual samples. This is evident from \label{tab:TCYC} and a manifold score that is 3 times larger than CVF.
\begin{figure}[h]
  \centering
  \includegraphics[width=\textwidth]{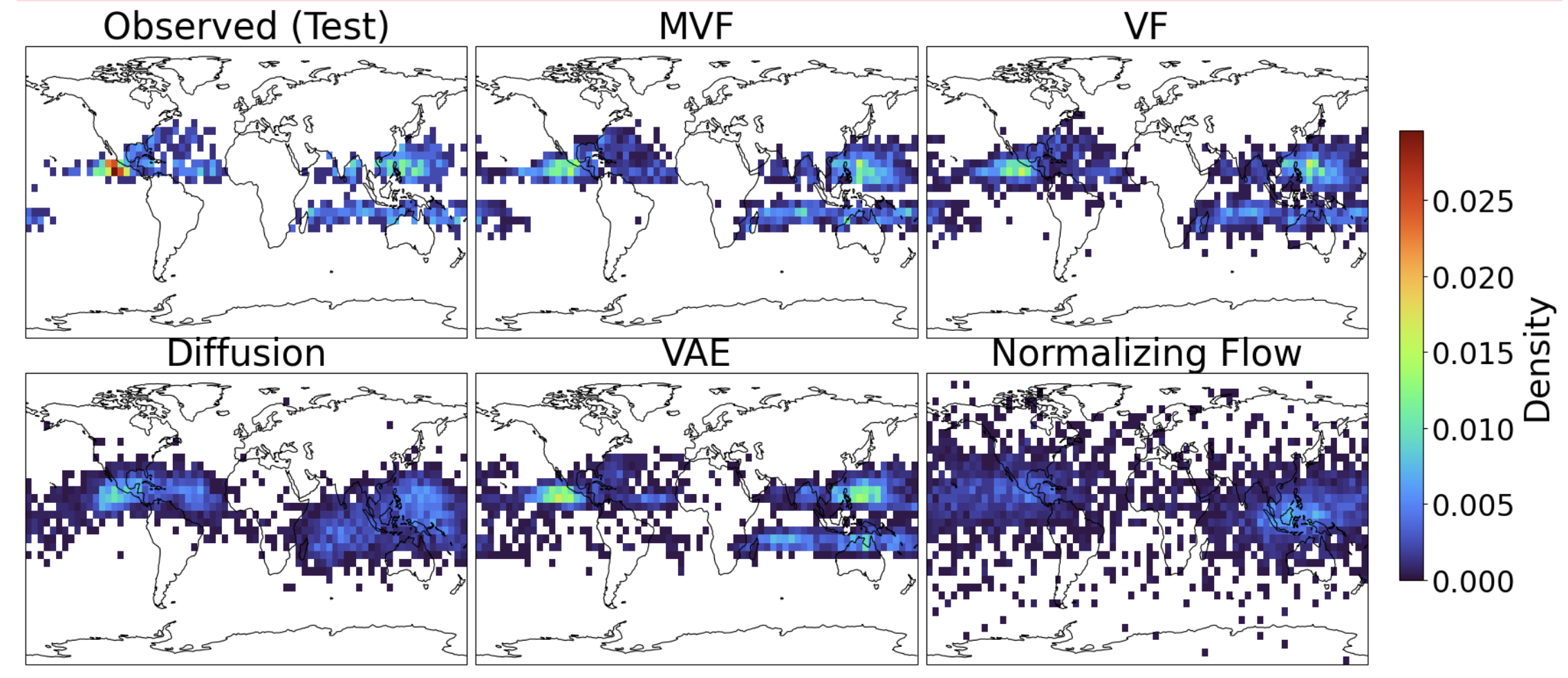}
  \caption{2025 predicted distribution of 5 methods with all observed tropical cyclones from 2015-2025}
  \label{fig:CoastTrop}
\end{figure}

This result is achieved by smoothing over the incorrect regions shifting the predictive distribution away from the tropical cyclone basins (high MD). Thus, while diffusion is statistically the most accurate, much of its predictive mass is invalid, and not usable in practice. conversely, by imposing a manifold constraint, CVF is able to reduce MD compared to each method at only a small cost to ED. 

Figure \ref{fig:CoastTrop} shows an example density from each method (Year 2025). The CVF density (\ref{fig:CoastTrop}b) is highly concentrated in the observed tropical cyclone basins, while VF is too dispersed with support far beyond the observed tropical cyclones (\ref{fig:CoastTrop}a). The other sampling methods all over support tropical cyclone genesis locations, particularly over the mid Atlantic and South America.

\subsection{Earthquake Events} \label{sec:earthquakes}

Finally, we apply our method to earthquake events in California. Our data comes from the ANSS Comprehensive Earthquake Catalog \citep{usgs_comcat_2026}, which we restrict to only magnitude $\ge4.0$ events. Unlike tropical cyclones, which have a soft manifold constraint, based on tropical cyclone basins, earthquakes are hard restricted to fault lines. Ensuring that earthquake UQ obeys this constraint is, therefore, essential.
As in Section \ref{sec:tropical_cyclones} we use a cumulative 2D empirical histogram predictor based on the forecasting model\citep{Werner_2010} up to the year 1985 where we start calibration in \ref{apx:EQ}. We compare VF and CVF against the same five baselines in Section \ref{sec:tropical_cyclones} again using ED and MD. We train the prediction model on 1976-1984 data, calibrate on 1985-2015 data, and test on 2016-2025 data.

\begin{wraptable}[9]{r}{0.55\columnwidth}
    \vspace{-1.5\baselineskip}
    \captionsetup{font=small}
    \caption{California earthquake energy distance metrics}
    \centering
    \setlength{\tabcolsep}{2.75pt}
    \begin{tabular}{lccccccc}
    \toprule
    Metric & M1 & M2 & M3 & M4 & M5 & VF & CVF \\
    \midrule
    \textbf{ED}\(\times 100\) & 1.69 & 1.56 & 1.98 & 17.50 & 2.09 & 0.54 & 0.47 \\
    {\normalfont\small\hspace{1em}Pred--Pred} & 0.00 & 0.00 & 0.00 & 0.01 & 0.01 & 0.04 & 0.02 \\
    {\normalfont\small\hspace{1em}True--Pred} & 0.02 & 0.02 & 0.02 & 0.10 & 0.03 & 0.04 & 0.02 \\
    \textbf{MD}\(\times 100\) & 0.07 & 0.15 & 0.64 & 9.03 & 0.24 & 1.79 & 0.19 \\
    \bottomrule
    \end{tabular}
    \label{tab:QuakeRes}
\end{wraptable}

Table~\ref{tab:QuakeRes} shows the average ED and MD metrics over the test period. Among all the methods, CVF and VF achieve the best overall performance, with the lowest ED with $(0.47)$ and $(0.54)$, respectively. However, CVF obtains a lower MD $(0.19)$ than all but M1, and M2 since they simply pull regions straight from the cumulative histogram model.

We note that CVF adequately improves upon VF by improving manifold score by almost a factor of 10 in addition to providing the optimal ED. We also note that M2 is a conformal approach, so our improvements in both ED and MD from its counterpart are not due to conformal prediction in general, but rather, the way in which we conformalize uncertainty.
\section{Discussion}

We introduced a new manifold-constrained conformal inference procedure for spatial event clouds. We represent each event cloud by its empirical measure which allows us to score predicted event clouds, or probability maps, against observed event clouds with (spherical) sliced Wasserstein distance. We then introduced an empirical manifold constraint that restricts our conformal sets to only those geometrically supported near the training data and derive a corresponding coverage lower bound. Finally, we derived a joint velocity field that allowed us to sample from the intersected sets and represent the predictive distribution as an ensemble. Empirical results on synthetic and geophysical experiments show that our method achieves substantially lower energy distance and manifold distance than standard baselines across a range of processes. Future work may include learning more robust manifold estimators or incorporating explicit physical constraints. 

\section{Acknowledgments}

This work was supported by the Laboratory Directed Research and Development program at Sandia National Laboratories, a multimission laboratory managed and operated by National Technology and Engineering Solutions of Sandia LLC, a wholly owned subsidiary of Honeywell International Inc. for the U.S. Department of Energy’s National Nuclear Security Administration under contract DE-NA0003525.
This paper describes objective technical results and analysis. Any subjective views or opinions that might be expressed in the paper do not necessarily represent the views of the U.S. Department of Energy or the United States Government.

\bibliography{main}
\bibliographystyle{iclr2026_conference}

\appendix
\section{Derivation of the Velocity Field} \label{apx:velocity_field}

We derive the minimum-norm velocity field used in Section 3.3. Fix \(X_{n+1}\) and write
\[
S(y)=S(Y,f(X_{n+1})), \qquad d(y)=d_M(Y),
\]
where \(y\in\mathbb R^{3|Y|}\) is the vectorized point cloud. The desired scalar dynamics are
\[
\frac{d}{dt}S(y(t))=-\lambda_1\{S(y(t))-\tau_\alpha\}, \qquad
\frac{d}{dt}d(y(t))=-\lambda_2\{d(y(t))-\varepsilon\}_+ .
\]
If \(v(y)=y'(t)\), then the chain rule gives
\[
\frac{d}{dt}S(y(t))=\nabla S(y)^\top v(y), \qquad
\frac{d}{dt}d(y(t))=\nabla d(y)^\top v(y).
\]
Thus the velocity must satisfy the two linear constraints
\[
\nabla S(y)^\top v(y)=-\lambda_1\{S(y)-\tau_\alpha\}, \qquad
\nabla d(y)^\top v(y)=-\lambda_2\{d(y)-\varepsilon\}_+ .
\]
Equivalently, with
\[
J(y)=
\begin{pmatrix}
\nabla S(y)^\top\\
\nabla d(y)^\top
\end{pmatrix},
\qquad
c(y)=
\begin{pmatrix}
\lambda_1\{S(y)-\tau_\alpha\}\\
\lambda_2\{d(y)-\varepsilon\}_+
\end{pmatrix},
\]
these constraints become \(J(y)v(y)=-c(y)\).

Since \(J(y)\in\mathbb R^{2\times 3|Y|}\), the system is generally underdetermined. We select the minimum-norm solution,
\[
v^\star(y)=\arg\min_{v\in\mathbb R^{3|Y|}}\frac12\|v\|_2^2
\quad\text{subject to}\quad J(y)v=-c(y).
\]
The Lagrangian is \(\mathcal{L}(v,\xi)=\frac{1}{2} v^\top v+\xi^\top\{J(y)v+c(y)\}\), where \(\xi\in\mathbb{R}^2\) is the vector of Lagrange multipliers. Stationarity with respect to \(v\) gives \(\partial \mathcal{L}/\partial v = v+J(y)^\top\xi=0\), hence \(v=-J(y)^\top\xi\). Substituting back into the constraint \(J(y)v=-c(y)\) gives
\[
-J(y)J(y)^\top\xi=-c(y),
\qquad\text{so}\qquad
J(y)J(y)^\top\xi=c(y).
\]
When \(J(y)\) has full row rank (i.e.\ \(\nabla_y S\) and \(\nabla_y d\) are linearly independent), \(J(y)J(y)^\top \in \mathbb{R}^{2\times 2}\) is invertible, so \(\xi=\{J(y)J(y)^\top\}^{-1}c(y)\), and therefore
\[
v^\star(y)=-J(y)^\top\{J(y)J(y)^\top\}^{-1}c(y).
\]
In the rank-deficient case, \(\nabla_y S\) and \(\nabla_y d\) are parallel, so \(J(y)J(y)^\top\) is singular. The constraint \(J(y)v=-c(y)\) is then either inconsistent (if \(c(y)\) is not in the column space of \(J(y)J(y)^\top\)) or has infinitely many solutions. Among the minimum-norm solutions in the consistent case, we replace the inverse with the Moore--Penrose pseudoinverse:
\[
v^\star(y)=-J(y)^\top\{J(y)J(y)^\top\}^{+}c(y).
\]
This recovers the full-rank solution when \(J(y)\) has full row rank, and otherwise projects \(c(y)\) onto the column space of \(J(y)J(y)^\top\) before solving, giving the minimum-norm least-squares velocity.

The full-rank expression satisfies the prescribed scalar dynamics exactly, since
\[
J(y)v^\star(y)
=
-J(y)J(y)^\top\{J(y)J(y)^\top\}^{-1}c(y)
=
-c(y).
\]
Consequently,
\[
\frac{d}{dt}S(y(t))=-\lambda_1\{S(y(t))-\tau_\alpha\}, \qquad
\frac{d}{dt}d(y(t))=-\lambda_2\{d(y(t))-\varepsilon\}_+ .
\]
For numerical stability we use the damped form
\[
v_\eta(y)=-J(y)^\top\{J(y)J(y)^\top+\eta I\}^{-1}c(y),
\]
with \(\eta>0\). This is the regularized version of the same minimum-norm controller and converges to the exact expression as \(\eta\downarrow 0\) whenever \(J(y)J(y)^\top\) is nonsingular.

\section{Proof of Proposition \ref{prop:manifold}} \label{apx:proof_manifold}

\begin{proof}
Let $S_{n+1}=S(Y_{n+1},f(X_{n+1}))$. Since
\[
\calC(X_{n+1})=\{Y\in\cY: S(Y,f(X_{n+1}))\le \taua,\ \dM(Y)\le \varepsilon\},
\]
if $Y_{n+1}\notin \calC(X_{n+1})$, then either $S_{n+1}>\taua$ or $\dM(Y_{n+1})>\varepsilon$. Therefore
\[
\mathbb{P}\bigl(Y_{n+1}\notin \calC(X_{n+1})\bigr)
\le
\mathbb{P}(S_{n+1}>\taua)+\mathbb{P}\bigl(\dM(Y_{n+1})>\varepsilon\bigr).
\]
Because $\cD_0$ is independent of the calibration sample, $\taua$ is the usual split conformal cutoff based on the exchangeable calibration scores $S_i=S(Y_i,f(X_i))$, so $\mathbb{P}(S_{n+1}>\taua)\le \alpha$.

Therefore, we only need to bound $\mathbb{P}(\dM(Y_{n+1})>\varepsilon)$. Write $E=\{\dM(Y_{n+1})>\varepsilon\}$. Because the distribution of $Y$ is strictly supported on $\cM$ and $A_1,\dots,A_N$ cover $\cM$, whenever $E$ occurs we still must have $Y_{n+1} \in A_k$ for some $k$. However, if we have $Y_i^0\in A_k$ for some $i$, then the diameter assumption on $A_k$ means $\SSW(\mu_{Y_{n+1}},\mu_{Y_i^0})\le \varepsilon$, and so $\dM(Y_{n+1})\le \varepsilon$, which is a contradiction. 

Thus, on the event $E\cap\{Y_{n+1}\in A_k\}$, none of the training points $Y_1^0,\dots,Y_{n_0}^0$ can also lie in $A_k$. Summing over all covering sets $A_k$ gives us the union bound
\[
\mathbb{P}(E)
\le
\sum_{k=1}^N \mathbb{P}\bigl(Y_{n+1}\in A_k,\ Y_i^0\notin A_k \text{ for all } i=1,\dots,n_0\bigr).
\]

By assumption all of the targets $Y_1^0,\dots,Y_{n_0}^0,Y_{n+1}$ are i.i.d. For each $k$, the probability above is
\[
\mathbb{P}(Y\in A_k)\bigl(1-\mathbb{P}(Y\in A_k)\bigr)^{n_0}.
\]
Since $\mathbb{P}(Y\in A_k)\ge \beta_\varepsilon$ for every $k$, this is at most $\mathbb{P}(Y\in A_k)(1-\beta_\varepsilon)^{n_0}$. Summing over $k$ gives
\[
\mathbb{P}(E)
\le
(1-\beta_\varepsilon)^{n_0}\sum_{k=1}^N \mathbb{P}(Y\in A_k)
\le
N(1-\beta_\varepsilon)^{n_0}.
\]
Combining the two bounds yields
\[
\mathbb{P}\bigl(Y_{n+1}\in \calC(X_{n+1})\bigr)
\ge
1-\alpha-N(1-\beta_\varepsilon)^{n_0}.
\]
\end{proof}

\section{Density-Valued Predictions}
\label{apx:density}
 
The conformity score $S(Y, f(X)) = \SSW(\mu_Y, f(X))$ requires both the observed target and the model prediction to be representable as elements of $P_2(\Sp^2)$. For observed targets this is immediate since any finite point cloud $Y$ maps to its empirical measure $\mu_Y$ as defined above. For predictions the situation is more varied. Many forecasting systems do not produce point cloud samples but instead output a continuous probability density $p : \Sp^2 \to \R_+$ over event locations, for example as a spatial intensity map from a neural network, a kernel density estimate, or a gridded output from a climate model. In all such cases the prediction is already an element of $P_2(\Sp^2)$, and the SSW score applies directly once the density is discretized to be compatible with the sliced Wasserstein computation.
 
We discretize the predicted density onto a regular latitude-longitude grid. Parameterize $\Sp^2$ by longitude $\lambda \in [-\pi, \pi)$ and latitude $\varphi \in (-\pi/2, \pi/2)$ via $x(\lambda, \varphi) = (\cos\varphi\cos\lambda,\, \cos\varphi\sin\lambda,\, \sin\varphi)$, with surface area element $d\sigma = \cos\varphi\, d\lambda\, d\varphi$. Fix grid edges $\{\lambda_i\}_{i=0}^{N_\lambda}$ and $\{\varphi_j\}_{j=0}^{N_\varphi}$ partitioning the domain, and define rectangular cells
\begin{equation}
    B_{ij} = \{x \in \Sp^2 : \lambda_i \leq \lambda(x) < \lambda_{i+1},\;
    \varphi_j \leq \varphi(x) < \varphi_{j+1}\}.
\end{equation}
The probability mass assigned to bin $(i,j)$ is
\begin{equation}
  q_{ij} = \int_{B_{ij}} p(x)\, d\sigma(x) \approx p(c_{ij})\,
  \cos\bar\varphi_j\,\Delta\lambda\,\Delta\varphi,
\end{equation}
where $c_{ij} = x(\bar\lambda_i, \bar\varphi_j)$ is the bin center at midpoint coordinates $\bar\lambda_i = (\lambda_i + \lambda_{i+1})/2$, $\bar\varphi_j = (\varphi_j + \varphi_{j+1})/2$, and the right-hand approximation holds when $p$ is smooth relative to the grid resolution. The weights $\{q_{ij}\}$ define a discrete measure
\begin{equation}
    \nu = \sum_{i,j} q_{ij}\, \delta_{c_{ij}} \in P_2(\Sp^2),
    \label{eq:quantized_measure}
\end{equation}
which is a valid element of the Wasserstein space $P_2(\Sp^2)$ and converges weakly to the continuous measure $p\,d\sigma$ as the grid is refined. The SSW distance between $\nu$ and any empirical measure $\mu_Y$ is then computed exactly as in the score definition, with the quantized measure replacing $f(X_{n+1})$.
 
To draw a point cloud realization from $\nu$ for use in the constrained flow, we sample bin indices $(i,j) \sim \mathrm{Categorical}(\{q_{ij}\})$ independently $m$ times and returns the corresponding centers. The resulting empirical measure converges to $\nu$ as $m \to \infty$. In practice $m$ is chosen to match the typical cardinality of observed clouds in the calibration set, so that the conformity scores from density-valued predictions are on the same scale as those from point cloud predictions.

\section{Competing Baseline Regions}
\label{apx:baselines}

Note that predictions for each of our numerical results, tropical cyclones, and earthquake models have outputs as discrete probability grids that we represent by
\[
P = (P_{ij}) \in [0,1]^{n_\phi \times n_\lambda},
\qquad
\sum_{i=1}^{n_\phi}\sum_{j=1}^{n_\lambda} P_{ij} = 1,
\]
defined on a latitude/longitude grid on \(S^2\). Let
\[
A = (A_{ij})
\]
denote the cell area weights, with \(A_{ij} > 0\). For an events at location \((\phi,\lambda)\), let \((i^\star,j^\star)\) be the nearest grid cell.

\paragraph{HDR Method} The HDR \citep{Hyndman1996HDR} is obtained by sorting all cells by forecasts \(P_{ij}\) in decreasing order and only keeping the smallest collection of cells whose cumulative probability is at least \(1-\alpha\).

Let
\[
\pi(1),\pi(2),\dots,\pi(m)
\]
be an ordering of the \(m=n_\phi n_\lambda\) grid cells such that
\[
P_{\pi(1)} \ge P_{\pi(2)} \ge \cdots \ge P_{\pi(m)}.
\]
Define
\[
K_{\mathrm{HDR}}(\alpha)
=
\min\left\{
k:\sum_{\ell=1}^k P_{\pi(\ell)} \ge 1-\alpha
\right\}.
\]
Then we have the HDR region
\[
R_{\mathrm{HDR},\alpha}
=
\{\pi(1),\dots,\pi(K_{\mathrm{HDR}}(\alpha))\}.
\]

\paragraph{Minimum area (M1)} The minimum area region ranks cells by probability mass per unit area. This method was influenced by Bayesian work \citep{BoxTiao1973}

\[
D_{ij} = \frac{P_{ij}}{A_{ij}},
\]

then as in HDR, keep the smallest number of cells with cumulative probability of at least $1-\alpha$

Let
\[
\pi_A(1),\dots,\pi_A(m)
\]
be an ordering such that
\[
D_{\pi_A(1)} \ge D_{\pi_A(2)} \ge \cdots \ge D_{\pi_A(m)}.
\]
Then define
\[
K_{\mathrm{MA}}(\alpha)
=
\min\left\{
k:\sum_{\ell=1}^k P_{\pi_A(\ell)} \ge 1-\alpha
\right\},
\]
and the minimum-area region is
\[
R_{\mathrm{MA},\alpha}
=
\{\pi_A(1),\dots,\pi_A(K_{\mathrm{MA}}(\alpha))\}.
\]

This differs from HDR only by incorporating area into the equation. This method is identical if all cells have equal area.

\paragraph{Conformal HDR Min Area (M2)}The conformal minimum area construction uses density score
\[
D_{ij}=\frac{P_{ij}}{A_{ij}}.
\]
For an observed event in cell \((i,j)\), define the conformal score
\[
S^{\mathrm{MA}}(i,j;P,A)
=
\sum_{r,s} P_{rs}\,\mathbf 1\{D_{rs} > D_{ij}\}.
\]

With calibration scores
\[
S_t^{\mathrm{MA}} = S^{\mathrm{MA}}(i_t,j_t;P_t,A),
\qquad t=1,\dots,n,
\]
and conformal quantile \(\hat q_\alpha\), the conformal minimum-area region is
\[
R_{\mathrm{cMA},\alpha}
=
\{(i,j): S^{\mathrm{MA}}(i,j;P_{n+1},A) \le \hat q_\alpha\}.
\]

Note that conformal minimum area region is simply just the area weighted version of conformal HDR.

\paragraph{Diffusion (M3)}

We provide 3 generative baselines for point cloud generation. One drawback is that generative baselines we provide have no probabilistic guarantee that conformal inference and HDR methods give. 

We implement a conditional denoising diffusion probabilistic model following the DDPM formulation \cite{Diffusion} directly on point clouds in $\mathbb{R}^3$ with a Gaussian forward noising process applied to the target cloud. We use a conditional denoiser to predict injected noise applied to the target cloud, the forecast cloud, and the diffusion time step. Conditioning is simply implemented with a permutation invariant encoder of the forecast cloud and condition dropout during training. Reverse sampling is performed in $\mathbb{R}^3$ with a projection back to the sphere at the end of the algorithm. We include a small Chamfer loss on the cloud estimate as well.

\paragraph{Spherical Flow (M4)} We generate a normalizing flow for the sphere by training a continuous velocity model directly on point clouds. Each sample is a cloud with $n$ points. To generate training pairs, we sample an initial cloud $X_0$ uniformly on the sphere and interpolate from $X_0$ to $X_1$ using spherical linear interpolation. A neural velocity field $v_\theta(x,t)$ is taken as an input and the output is a tangent vector on the sphere. The network uses two hidden layers of width 128 with tanh nonlinearities, and the output is projected onto the tangent space at $x$. The training target are the difference in velocity along the linear interpolation projected back into the tangent space. The loss minimized MSE between the predicted and target tangent velocities. We train for 200 epochs using Adam with learning rate $10^{-3}$. We sample an initial cloud uniformly on the sphere and integrate the learned velocity field for 100 Euler steps renormalizing points to the sphere if necessary after each step.

\paragraph{Variational Auto Encoder (VAE)(M5)} 

We also utilize a forecast conditioned VAE directly on point clouds. The VAE uses permutation invariant point cloud encoders for the forecast and target clouds, a latent Gaussian representation, and a residual decoder that corrects the forecast cloud. Training uses a Chamfer reconstruction term, a KL penalty, and small repulsion regularizer to encourage diverse sampling and preventing mode collapse.

\section{Simulated Point Processes}
\label{apx:PP}

\subsection{Homogeneous Poisson Process}

\begin{figure}[h]
  \centering
  \includegraphics[width=\textwidth]{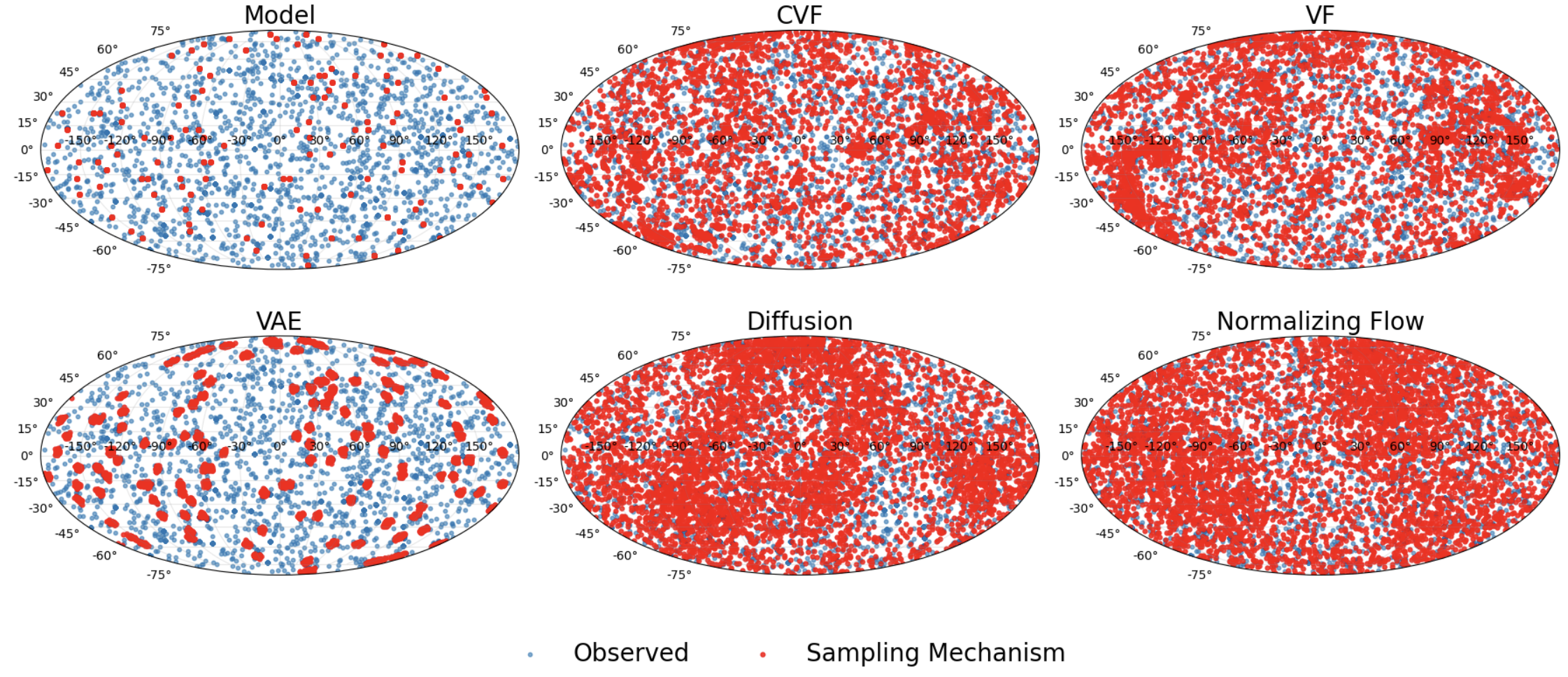}
  \captionsetup{font=small}
  \caption{Homogenous Poisson model with generated sample prediction}
  \label{fig:examples}
\end{figure}

We simulate a homogeneous Poisson process on the sphere with \[N \sim Poisson(4\pi \lambda)\] where $\lambda > 0$ is a constant intensity per unit area and $N$ is the total count. We condition on $N$ and points are i.i.d. uniform on the sphere. In our experiment, we use $\lambda = 20$

\subsection{Misspecified Spiral}

\begin{figure}[h]
  \centering
  \includegraphics[width=\textwidth]{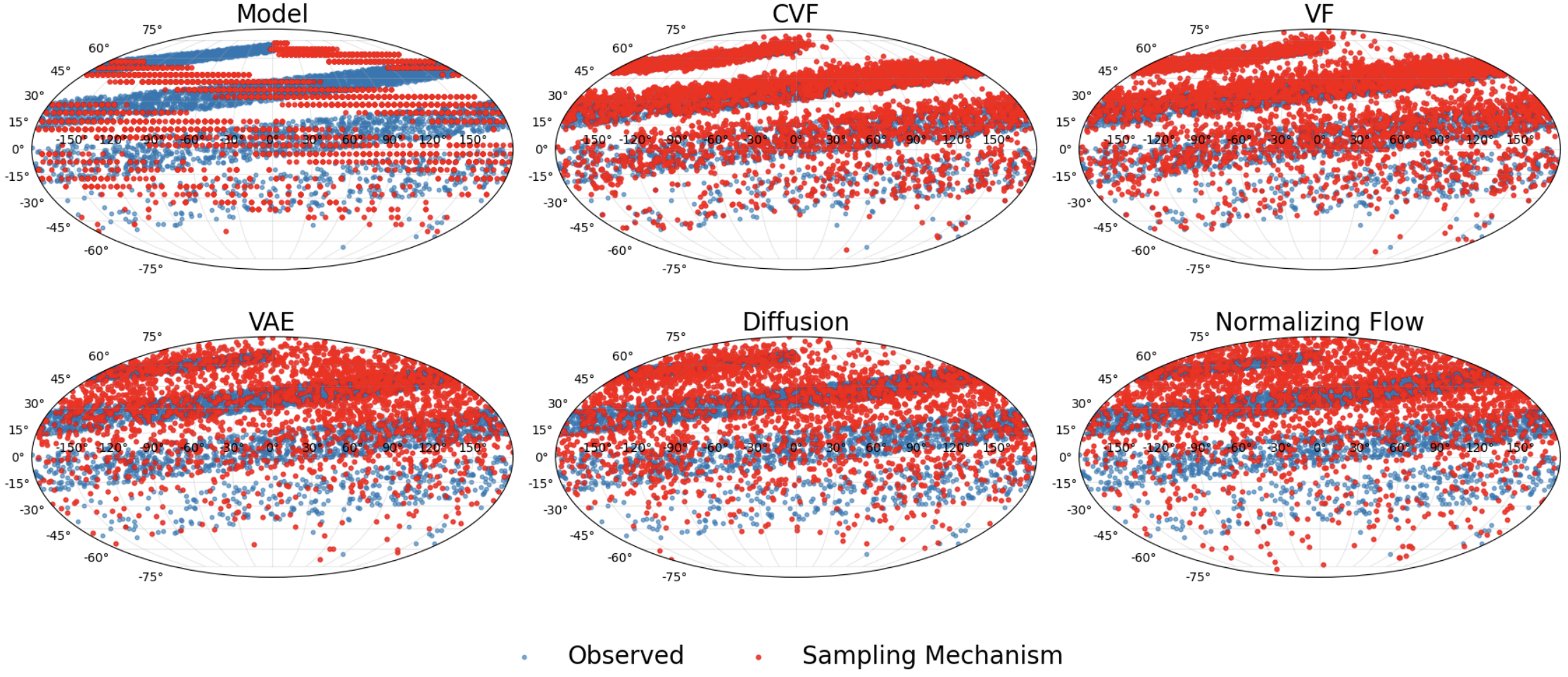}
  \captionsetup{font=small}
  \caption{Spiral Simulation with opposite direction oracle model fit (Top Left)}
  \label{fig:examples}
\end{figure}

The spiral simulation is built on the sphere using a spiral curve structure. This simulation also has a misspecified model where the spiral spins the opposite direction and has higher rate near the top of the sphere. To generate points along the spiral, the parameter $t$ is sampled from a beta distribution with parameters $\alpha = 3$ and $\beta = 1.2$. This process places points on expanding bands as they travel to the bottom of the sphere. Each replicate in the spiral contains a true cloud of 100 points and has 50 timepoints to mimic smaller data with more difficult manifold structure.

\subsection{Inhomogeneous Poisson Processes}

\begin{figure}[h]
  \centering
  \includegraphics[width=\textwidth]{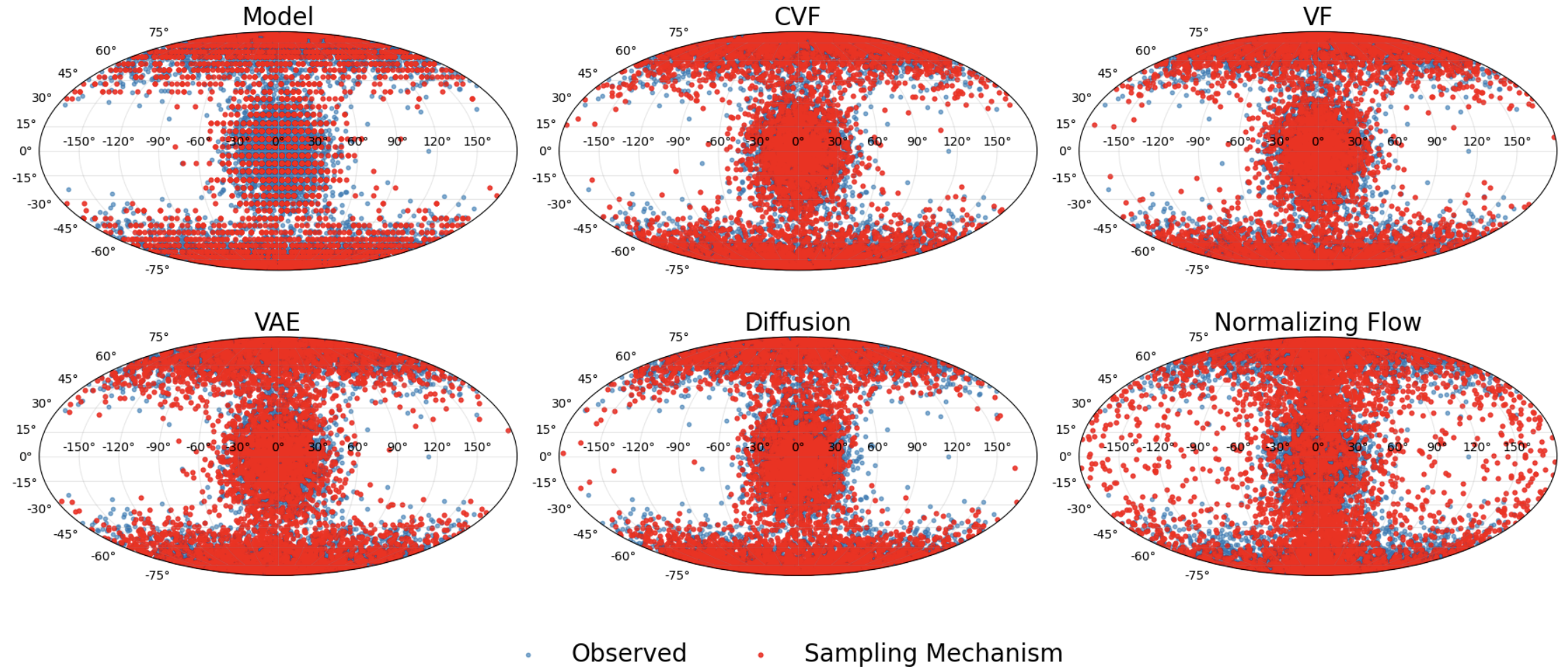}
  \captionsetup{font=small}
  \caption{Inhomogenous Poisson with 3 modes model with oracle model fit (Top Left)}
  \label{fig:examples}
\end{figure}

Another important aspect that our method must show is high performance for point processes that have structured, non-uniform behavior unlike homogeneous Poisson processes \citep{baddeley2015spatial}. We consider two inhomogeneous with differing intensity functions. One to represent a single large mode on the globe, and another to represent three different modes on the globe. Both will utilize aspects of the von Mises Fisher distribution which is widely used in spatial statistics on the globe \citep{banerjee2005clustering}.

We define IHPP-3 (3 hotspots) as an inhomogeneous Poisson process on a "belt" defined as $B(x) = e^{-(\frac{z}{\sigma})^2}$ placing the centers at the central meridian and three modes defined as $H(x) = \sum_m^Me^{\kappa_h\mu_m'x}$ This leads to the defined intensity function \[\lambda(x) = \lambda_0(1+be^{\frac{-z^2}{\sigma_b^2}} + a_h \sum_m^Me^{\kappa_h\mu_m'x})\] To sample we again use thinning from the same uniform distribution as the 1 hotspot model with the conserveative upper bound of $\lambda_{max}$ We utilize parameters \[\lambda_0 = 0.01, \ b = 0.6\ \sigma_b = 0.3,\ M = 3,\ a_h = 1, \ \kappa_h = 10\]

\subsection{Hawkes Process}

\begin{figure}[h]
  \centering
  \includegraphics[width=\textwidth]{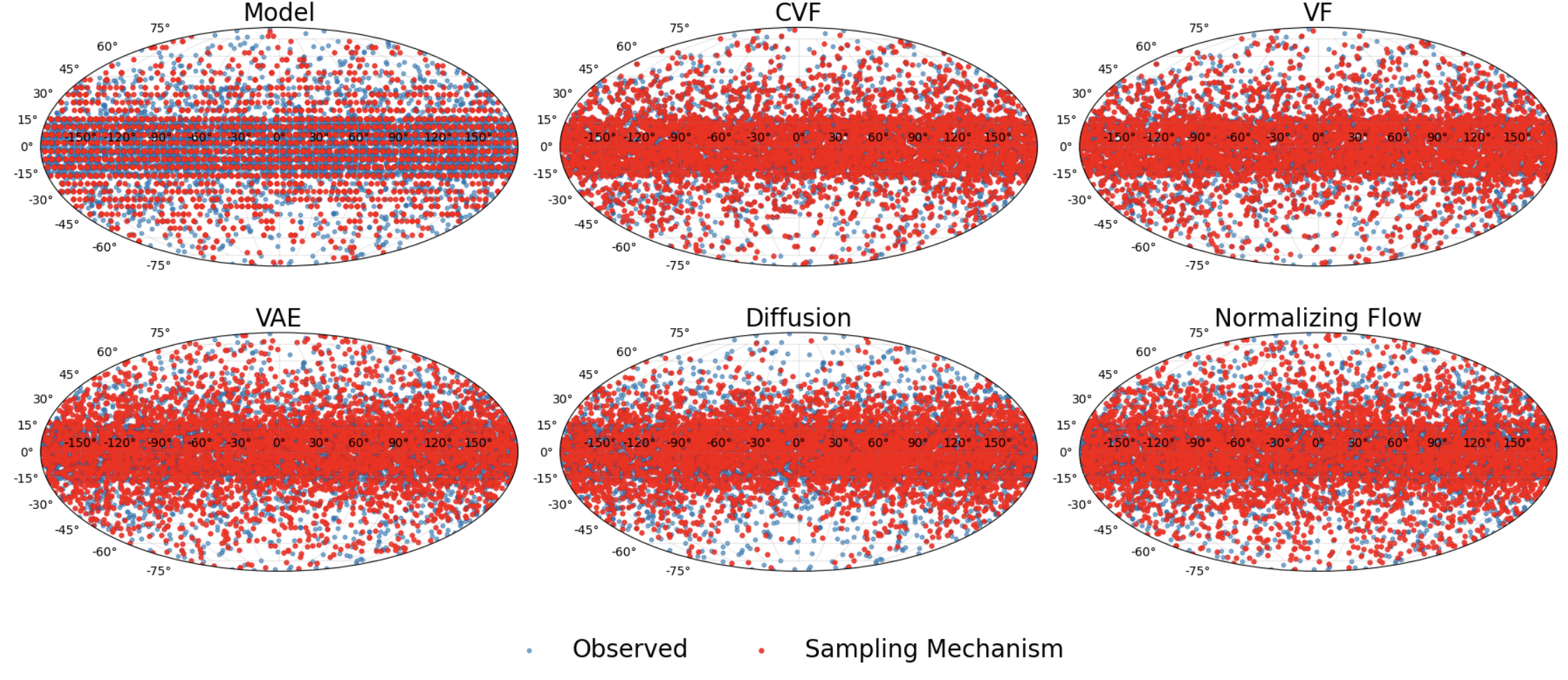}
  \captionsetup{font=small}
  \caption{Hawkes process on a belt with oracle model fit (Top Left)}
  \label{fig:examples}
\end{figure}

Hawkes processes are "self-exciting" point processes that were designed to have inhomogeneous Poisson process randomness, but to have generated points have larger probabilities to generate more points \citep{hawkes1971spectra}. Specifically in Earthquake modeling, the idea of utilizing Hawkes based processes for simulation or modeling is popular \citep{ogata1988statistical}. Due to our interest in recreation of structures occurring due to a structure on a sphere, we create an equatorial belt utilizing a Hawkes process where events near the equator are more common and slowly taper out as the distance increases with a self-exciting property. 

We define Hawkes process $Hawkes$ on $[0,T]\times S^2$ with temporal Hawkes intensity 

\[\lambda(t|H_t) = \mu + \sum_{T_i < t}\alpha e^{-\beta(t-T_i)}; \ t \in [0, T]\]

We then include type probabilities \[P(immigrant|H_t) = \frac{\mu}{\lambda(t)}, \ P(Offspring|H_t) = \frac{c_j(t)}{\lambda(t)}\] In addition, with the same intensity, we define a "belt" region as $B = \{x \in S^2: |x_3| < z_{max}\}$ where $x = \{x_1, x_2, x_3\}$ and $z_{max} \in (0,1)$. We let $g_{belt}$ be uniform on the belt and the immigrant points have locations \[X|offspring \sim uniform(belt)\] and offspring as \[X|offspring \sim vMF(\mu=X_j, \kappa_{space})\]

We utilize parameters \[\mu = 0.7, \alpha = 0.6, \beta = 1.5, \kappa_{space} = 12, z_{max} = 0.3\]

\subsection{Tropical Cyclone Model}
\label{apx:ModCyc}

We construct a frozen presence background model for tropical cyclone genesis prediction. We utilize environmental data obtained from the Copernicus data store along with cyclone genesis data from the IBTrACS database for the first time a storm attains tropical cyclone status. For each year for the training period (1980-1994), the genesis locations are mapped onto the same forecast grid as a distribution. This is then smoothed and normalized to form a yearly target probability map for downstream analysis. 

The purpose of this model is to produce a probability distribution over the globe to estimate the spatial distribution of genesis locations. This fixed design will allow us to not violate exchangeability for conformal methods.

Let \[X_y \in \mathbb{R}^{C\times H \times W}\] denote the yearly environmental tensor for year $y$ by stacking 12 predictor fields across all variables. Let \[\mathcal{G}_y = \{(i_k, j_k)\}^{N_y}_{k=1}\] denote the set of observed genesis cells in year $y$. For each forecast year $y$ prior to 1995, we form a training dataset from all years $y' < y$. Any positive values are taken from cells in $\mathcal{G}_y$ while the negative are drawn from randomly sampled background cells. The vector associated with a grid cell is the channel vector \[X_{y'}[:,i,j]\] We fit a regularized logistic regression model on this dataset interpreting the fitted output as a relative occurrence intensity. 

We let $z_r = 1$ if row $r$ contained a genesis location or cell, and $z_r = 0$ if row $r$ corresponds with a background cell. We then fit $l_2$ regularized logistic regression with conditional probability  \[P(z_r = 1|\tilde{x}_r) = \frac{1}{1+exp(-(b+\tilde{x}_r'\beta))}\] The fitted parameters are simply obtained by minimizing the regularized logistic loss. This model is held out for all later forecast years. For each year, we evaluate the fitted model at every cell on the globe to obtain \[\hat p_y(i,j) = \frac{1}{1+exp(-(b+\tilde{x}_{y,i,j}'\beta))}\] and normalize to \[\hat\pi_y(i,j) = \frac{\hat p_y(i,j)}{\sum^H_{i=1}\sum^W_{j=1}\hat \pi_y(i,j)}\] as the forecast probability map for year $y$.

For conformal methods, we hold out the years 1995-2014 as calibration with no additional training to the model. SSW scores are recorded for VF and CVF methods. Our data typically has between 50 and 110 points a year.

\subsection{Earthquakes}
\label{apx:EQ}

\begin{figure}[h]
  \centering
  \includegraphics[width=\textwidth]{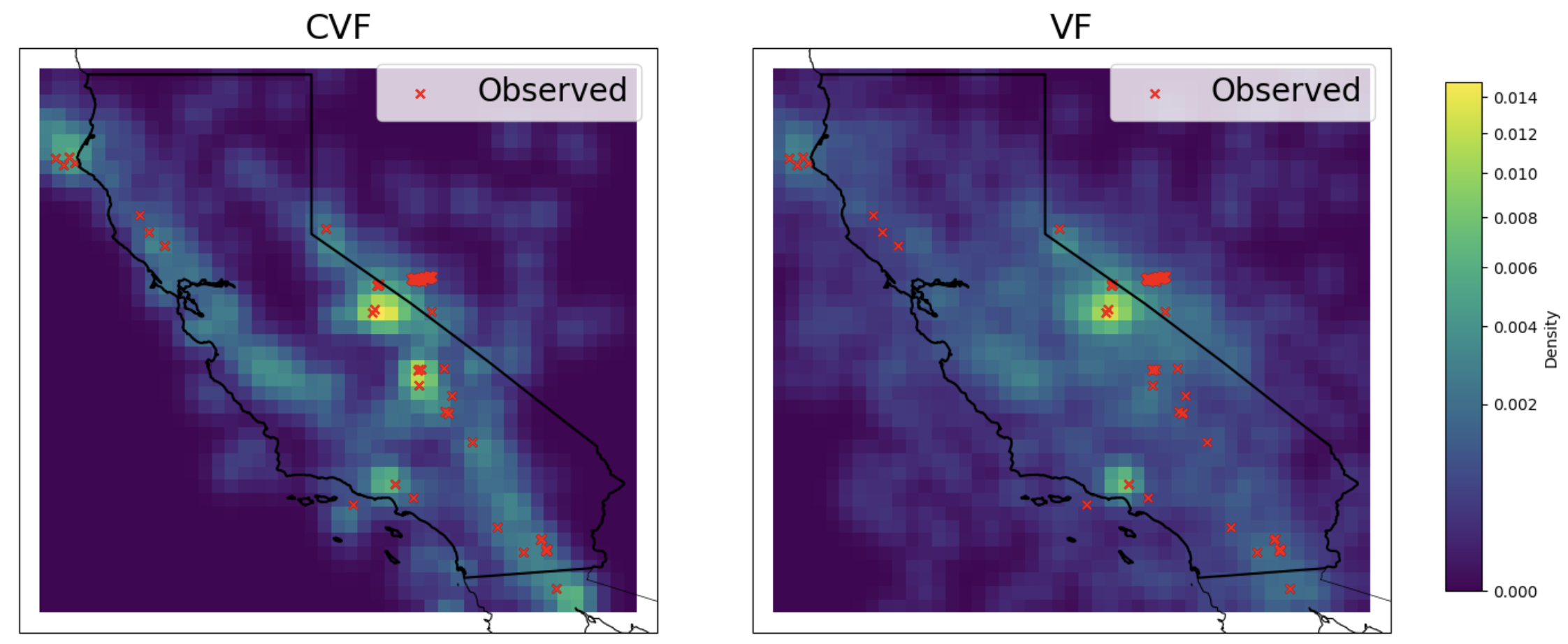}
  \captionsetup{font=small}
  \caption{Earthquake velocity field and constrained velocity field with observed values for target year 2024}
  \label{fig:examples}
\end{figure}

We model earthquake occurrences in the region of California using a grid spatial forecast model constructed from historical data. Let the past earthquake data be \[\mathcal{D} = \{(t_i, l_i^{lat}, l_i^{lon}, m_i, d_i)\}^N_{i=1}\] where $t_i$ is the time of the event, $(l_i^{lat}, l_i^{lon})$ is the earthquake location, $m_i$ is the magnitude, and $d_i$ is the depth of the earthquake. We restrict data to $m_i > 4.0$ in the California region, denoted by \[ \mathcal{R} = [l^{lon}_{min}, l^{lon}_max]\times [l^{lat}_{min},l^{lat}_{max}]\] We discretize this on a fixed latitude longitude grid and let each grid contain cells $\mathcal{C}_{jk}$ indexed by the latitude bin $j$ and longitude bin $k$ with center $c_{jk}$.

Instead of re-fitting the model each year, we instead construct a fixed spatial forecast from an initial training block in order to support the exchangeability assumption. We fit the model once using earthquakes from 1976 to 1984, use the years 1985 through 2015 for calibration, and evaluate on the 2016 through 2025 data. The observed earthquakes in each calibration or test year are then compared against samples from the fixed forecast grid.

The fitted forecasting model is simply a cumulative two dimensional histogram predictor. For each gridcell $\mathcal{C}_{jk}$, we count the number of historical earthquakes that fall into that cell \[N_{jk}(T) = \sum^N_{i=1} 1\{ t_i < T, m_i \geq 4.0, (l_i^{lat}, l_i^{lon})\in \mathcal{C}_{jk}\} \] $N_{jk}(T)$ defines an unnormalized spatial intensity which is then converted to a probability surface by normalization. This gives a discrete forecast distribution over the grid of California representing an empirical gridded version of an adaptive smoothed seismicity model typically used in forecasts. Our data typically has 11 to 140 points per yearly point cloud.

\subsection{Simulation Timing}

Spherical Sliced Wasserstein distance takes a longer time to compute, yet our method produces samples in a reasonable amount of time.

\begin{table}[h]
\centering
\captionsetup{font=small}
\caption{Mean integration time in seconds to integrate to every point cloud in the calibration set, across different numbers of points $n$, slices $L$, and integration steps. Standard deviations are shown in parentheses under 20 repetitions.}
\label{tab:timing_results}
\setlength{\tabcolsep}{5pt}
\begin{tabular}{ccc c}
\toprule
$n$ & $L$ & Steps & Time (s) \\
\midrule
50  & 25  & 10 & 30.27 (0.41) \\
50  & 25  & 20 & 49.13 (0.42) \\
50  & 50  & 10 & 38.68 (0.44) \\
50  & 50  & 20 & 64.37 (0.37) \\
50  & 100 & 10 & 43.23 (0.41) \\
50  & 100 & 20 & 71.58 (0.63) \\
\midrule
100 & 25  & 10 & 92.08 (0.59) \\
100 & 25  & 20 & 166.53 (0.37) \\
100 & 50  & 10 & 97.68 (0.44) \\
100 & 50  & 20 & 175.91 (0.51) \\
100 & 100 & 10 & 119.35 (0.72) \\
100 & 100 & 20 & 215.85 (0.92) \\
\bottomrule
\end{tabular}
\end{table}

\subsection{Variance and Sample Diversity}

We now evaluate the effectiveness of sampling via the joint velocity field (Equation \ref{eq:velocity}) for representing the manifold constrained prediction sets in practice (\ref{eq:empirical_manifold_region}). Figure \ref{fig:VAR} shows the distribution of SSW distances between actual point clouds at the 90th quantile and between our models samples at 90th quantile and the actual data.
\begin{figure}[h]
  \centering
  \includegraphics[width=\textwidth]{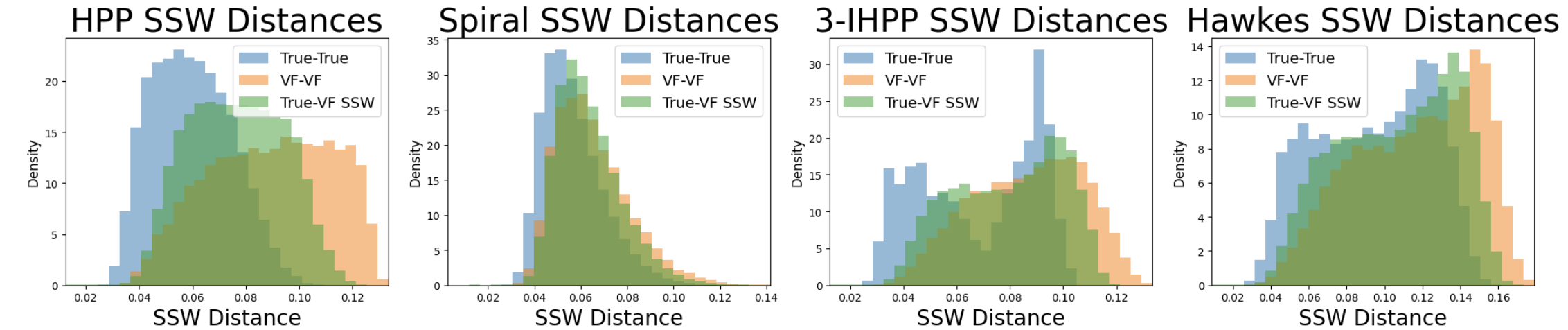}
  \captionsetup{font=small}
  \caption{Distribution of intra-sample distances from CVF method ($\alpha = 0.1$) compared with inter-sample distances of the data and the CVF method (90th percentile)}
  \label{fig:VAR}
\end{figure}
In each case the model is able to simulate point clouds that have similar variability with the data as the data does with itself. This shows that our velocity field is not simply collapsing to a subspace, but, in fact, sampling close to the true conditional distribution at the 90th quantile. In addition, the velocity field is sampling a similar distributional spread between each of its samples.

\end{document}